\documentclass[conference]{IEEEtran}
\IEEEoverridecommandlockouts
\usepackage{amsmath}
\usepackage{mathtools}
\usepackage{algorithm}
\usepackage{algorithmic}
\usepackage{amssymb}
\usepackage{amsthm}
\usepackage{stmaryrd}
\usepackage{graphicx}    
\usepackage{subcaption}  
\usepackage{caption}
\usepackage[capitalize,noabbrev]{cleveref}
 
\newtheorem{theorem}{Theorem} 
\newtheorem{remark}{Remark} 
\newtheorem{lemma}{Lemma} 
\usepackage{booktabs}
\usepackage{array}
\usepackage{enumitem}
\newcommand{\ie}{{\em i.e.,~}}
\newcommand{\eg}{{\em e.g.,~}}

\usepackage{cite}
\usepackage{amsmath,amssymb,amsfonts}
\usepackage{algorithmic}
\usepackage{graphicx}
\usepackage{textcomp}
\usepackage{xcolor}
\def\BibTeX{{\rm B\kern-.05em{\sc i\kern-.025em b}\kern-.08em
    T\kern-.1667em\lower.7ex\hbox{E}\kern-.125emX}}
\usepackage[numbers]{natbib}
\begin{document}

\title{Asynchronous Gossip Algorithms \\for Rank-Based Statistical Methods}

\author{\IEEEauthorblockN{Anna van Elst}
\IEEEauthorblockA{\textit{LTCI}, \textit{T\'el\'ecom Paris}
\\
Institut Polytechnique de Paris
\\
\texttt{\footnotesize anna.vanelst@telecom-paris.fr}}
\and
\IEEEauthorblockN{Igor Colin}
\IEEEauthorblockA{\textit{LTCI}, \textit{T\'el\'ecom Paris}
\\
Institut Polytechnique de Paris
\\
\texttt{\footnotesize igor.colin@telecom-paris.fr}}
\and
\IEEEauthorblockN{Stephan Cl\'emen\c{c}on}
\IEEEauthorblockA{\textit{LTCI}, \textit{T\'el\'ecom Paris}
\\
Institut Polytechnique de Paris
\\
\texttt{\footnotesize stephan.clemencon@telecom-paris.fr}}
}

\maketitle

\begin{abstract}
As decentralized AI and edge intelligence become increasingly prevalent, ensuring robustness and trustworthiness in such distributed settings has become a critical issue---especially in the presence of corrupted or adversarial data. Traditional decentralized algorithms are vulnerable to data contamination as they typically rely on simple statistics (\textit{e.g.}, means or sum), motivating the need for more robust statistics. In line with recent work on decentralized estimation of trimmed means and ranks, we develop gossip algorithms for computing a broad class of rank-based statistics, including L-statistics and rank statistics---both known for their robustness to outliers. We apply our method to perform robust distributed two-sample hypothesis testing, introducing the first gossip algorithm for Wilcoxon rank-sum tests. We provide rigorous convergence guarantees, including the first convergence rate bound for asynchronous gossip-based rank estimation. We empirically validate our theoretical results through experiments on diverse network topologies. 
\end{abstract}


\pagestyle{plain} 
\section{Introduction}

\setlength{\jot}{1pt}
The rise of the Internet of Things (IoT) has led to a surge in connected devices, generating vast amounts of data at the network's edge. At the same time, AI systems are becoming more complex, pushing the need for decentralized approaches. Edge intelligence—processing data close to where it is generated—offers a promising alternative to traditional, centralized AI. In this context, gossip algorithms have emerged as a lightweight, scalable way for devices to collaborate. These algorithms operate by having each node communicate only with its neighbors, making them ideal for environments like peer-to-peer systems or sensor networks, where data is naturally distributed and communication is limited. Gossip methods have been used for tasks such as computing averages, maximum, and sum \citep{boyd2006randomized, shah2009gossip, jelasity2005gossip}, and they also play a role in decentralized learning and optimization \citep{duchi2011dual, colin2016gossip}. However, their reliance on local communications makes them vulnerable to corrupted nodes---whether due to hardware failures or malicious attacks. To make gossip algorithms more robust, it is therefore necessary to go beyond simple statistics estimation.

In our recent work \cite{van2025robust}, we introduced \textsc{GoTrim}, a gossip algorithm for estimating robust means (\textit{i.e., }trimmed means), which leverages our \textsc{GoRank} algorithm for rank estimation. In this paper, we build on that foundation by developing gossip algorithms capable of computing a broad class of rank-based statistics. These include $L$-statistics (linear combinations of order statistics) and rank statistics (weighted sums of transformed ranks) \cite{van2000asymptotic}. This enables us to develop the first gossip algorithm for robust two-sample hypothesis testing. While there is a large body of literature on distributed Bayesian hypothesis testing using belief propagation \cite{alanyali2004distributed, olfati2006belief, soltanmohammadi2012decentralized}, to the best of our knowledge, no gossip-based algorithm exists for performing rank tests, such as the Wilcoxon rank-sum test \citep{HAJEK1999}. 

We also observe that \textsc{GoTrim} converges more slowly to the true trimmed mean when the trimming parameter $\alpha$ is large (\textit{i.e., }close to $0.5$). To address this, we introduce an enhanced version, \textit{Adaptive} \textsc{GoTrim} which provides more accurate estimates when the trimmed mean approaches the median. Additionally, motivated by our prior work showing that asynchronous communication (\textit{i.e.,} no global clock synchronization) is not only more practical in real-world situations but also yields slightly better performance, all our proposed gossip algorithms are designed for the asynchronous setting.

Another key contribution of this work is the convergence analysis of our gossip algorithms. We begin by analyzing \textsc{GoRank} in the asynchronous setting---a challenging setting that was omitted in our prior work. We then turn to the convergence analysis of our new gossip algorithms, focusing on the synchronous setting for clarity and simplicity. However, similar results can be extended to the asynchronous setting. 

Our main contributions are summarized as follows:

\noindent $\bullet$ We establish the first theoretical convergence rate for asynchronous gossip-based ranking on arbitrary communication graphs: we prove an \( \mathcal{O}(1/\sqrt{t}) \) convergence rate for the expected absolute error (where \( t\geq 1 \) denotes the number of iterations), matching the rate in the synchronous setting. \\
\noindent $\bullet$ We propose a novel gossip algorithm for estimating a wide class of rank-based statistics, enabling Wilcoxon rank-sum tests. To the best of our knowledge, this is the first robust gossip algorithm for two-sample distributed hypothesis testing. We prove a convergence rate of order \( \mathcal{O}(1/t) \) for Wilcoxon statistic estimation in the synchronous setting. \\
\noindent $\bullet$ We introduce \textit{Adaptive} \textsc{GoTrim}, designed to accelerate convergence when the trimmed mean is close to the median. We first establish an \( \mathcal{O}({1/\sqrt{t}}) \) bound on the expected absolute error for the original \textsc{GoTrim} in the synchronous setting, and obtain the same rate for \textit{Adaptive} \textsc{GoTrim}. \\
\noindent $\bullet$ Finally, we conduct experiments on various contaminated data distributions and network topologies.

This paper is organized as follows. Section 2 introduces the notation and problem formulation. In Section 3, we analyze the convergence of \textit{Asynchronous} \textsc{GoRank}. Section 4 presents a gossip algorithm for estimating rank-based statistics, along with its application to rank tests and a convergence analysis for the Wilcoxon statistic. Section 5 introduces \textit{Adaptive} \textsc{GoTrim} and provides its its theoretical convergence guarantees. Finally, Section 6 reports numerical experiments. 

\section{Preliminaries}

This section introduces the necessary notations, describes the setup, and formulates the problem of estimating rank-based statistics (\textit{e.g.}, rank statistics and $L$-statistics).

\paragraph{Notation} Let $n\geq 1$. We denote scalars by normal lowercase letters \(x \in \mathbb{R}\), vectors (identified as column vectors) by boldface lowercase letters \(\mathbf{x} \in \mathbb{R}^n\), and matrices by boldface uppercase letters \(\mathbf{X} \in \mathbb{R}^{n \times n}\). The set \(\{ 1, \dots, n \}\) is denoted by \([n]\), $\mathbb{R}_n$'s canonical basis by $\{\mathbf{e}_k:\; k\in [n]\}$, the indicator function of any event \(\mathcal{A}\) by \(\mathbb{I}_{\mathcal{A}}\), the transpose of any matrix \(\mathbf{M}\) by \(\mathbf{M}^{\top}\) and the cardinality of any finite set \(F\) by \(|F|\). By \(\mathbf{I}_n\) is meant the identity matrix in \(\mathbb{R}^{n \times n}\), by \(\mathbf{1}_n = (1,\; \ldots,\; 1)^{\top}\) the vector in \(\mathbb{R}^n\) whose coordinates are all equal to one, by \(\|\cdot\|\) the usual \(\ell_2\) norm, by \(\|\cdot\|_{\operatorname{op}}\) the operator norm, and by \(\mathbf{A}\odot \mathbf{B}\) the Hadamard product of matrices \(\mathbf{A}\) and \(\mathbf{B}\). We model a network of size \(n > 0\) as an undirected graph \(\mathcal{G} = (V, E)\), where \(V = [n]\) denotes the set of vertices and \(E \subseteq V \times V\) the set of edges. We denote by \(\mathbf{A}\) its adjacency matrix, meaning that for all \((i, j) \in V^2\), \([\mathbf{A}]_{ij} = 1\) iff \((i, j) \in E\), and by \(\mathbf{D}\) the diagonal matrix of vertex degrees. The graph Laplacian of \(\mathcal{G}\) is defined as \(\mathbf{L} = \mathbf{D} - \mathbf{A}\).

\paragraph{Setup} We study a decentralized setting involving $n \geq 2$ real-valued observations $X_1, \ldots, X_n$, each located at a distinct node of a communication network modeled by a \textit{connected} and \textit{non-bipartite} graph $\mathcal{G} = (V, E)$; see \cite{Diestelbook}. Specifically, node $k \in [n]$ holds observation $X_k$. To simplify the analysis, we assume there are no ties: $X_k \neq X_l$ for any $k \neq l$. Communication is pairwise and occurs randomly: at each step, an edge $e \in E$ is selected with probability $p_e$, and the two connected nodes exchange information. We operate in an \textit{asynchronous gossip} framework, where each node's local clock follows an independent Poisson process with rate 1. This setup is equivalent to a global Poisson clock of rate $n$, with a random edge activated at each tick, similar to the synchronous case (see \cite{boyd2006randomized} for clock modeling details). Under this framework, the probability that node $k$'s local clock ticks is $1/n$. Furthermore, each edge $e$ is activated at a tick with probability $p_e$. For simplicity, we assume a uniform selection among neighbors (see \cref{fig:async}), yielding
$$
p_e = \frac{1}{n} \left( \frac{1}{d_i} + \frac{1}{d_j} \right),  \, \text{with} \,  e = (i,j) \in E,
$$
where $d_i$ and $d_j$ denote the degrees of nodes $i$ and $j$, respectively. We also consider the presence of outliers: under Huber’s contamination model \citep{huber1992robust}, a fraction $0 < \varepsilon < 1/2$ of the observations may be corrupted.
\begin{figure}
    \centering
    \includegraphics[width=0.7\linewidth]{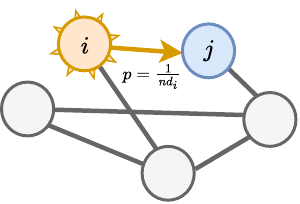}
    \caption{Visualization of the Asynchronous Gossip Framework. At each tick, node $i$ wakes up with probability $1/n$ and selects a neighbor uniformly at random (with probability $1/d_i$).}
    \label{fig:async}
\end{figure}

\paragraph{Decentralized Estimation of Rank-based Statistics} Our objective is to design a gossip-based algorithm that accurately estimates rank-based statistics. Denote $r_k$ is the rank of observation $X_k$, computed as $r_k = 1 + \sum_{l=1}^n \mathbb{I}_{\{X_k > X_l\}}$ in the absence of ties. The problem can be formulated as the following general expression
\begin{equation}
\label{eq:rank-stat}
T_n = \sum_{k=1}^n f(r_k)\, g(X_k),
\end{equation}
where $f$ is a function of the ranks $r_k$, and $g$ is a function of the observations $X_k$. By choosing $g$ as indicator functions of the observations and $f$ as appropriate transformations of the ranks, we recover rank statistics. A notable special case is the Wilcoxon statistic, used in the Wilcoxon rank sum test---a well-known robust nonparametric test for comparing two samples. Assume the dataset $S = \{X_1, \ldots, X_N\}$ is partitioned into two disjoint subsets, $S_1$ and $S_2$, \ie $S = S_1 \cup S_2$, with $|S_1| = n_1$ and $|S_2| = n_2$. The Wilcoxon statistic is given by
\begin{equation}
\label{eq:wilcoxon}
t_n = \sum_{k=1}^n r_k b_k,
\end{equation}
where $b_k = \mathbb{I}_{\{X_k \in S_1\}}$. This corresponds to a linear rank statistic with no transformation applied to the ranks. More general rank statistics apply a transformation to the ranks, \ie using $f(r_k)$ for some function $f$. For instance, the van der Waerden statistic employs $f = \Phi^{-1}$, the quantile function of the standard normal distribution.

For $g = \operatorname{id}$ (the identity function) and $f$ a function of the ranks, we obtain $L$-statistics. Typically, $f$ is an indicator function of the ranks (possibly scaled by a constant). For example, denoting \( m = \lfloor \alpha n \rfloor \), the trimmed mean is an $L$-statistic with function $f$ given by $f(r_k) = \mathbb{I}_{\{r_k \in I_{n, \alpha}\}}/(n - 2m) $ and inclusion interval \( I_{n, \alpha} = [m+1, n-m] \), see \cite{van2025robust}.

\section{Asynchronous Distributed Ranking}

In this section, we present the convergence analysis of \textit{Asynchronous} \textsc{GoRank}, as introduced in prior work \cite{van2025robust}. We show that expected absolute error decreases at a rate of \( \mathcal{O}(1/\smash{\sqrt{ct}}) \), where the constant $c > 0$ (defined in Theorem \ref{thm:async-gorank}) reflects the connectivity of the graph $\mathcal{G}$. A more connected graph yields a larger value of $c$, resulting in a tighter (\ie smaller) upper bound on the convergence rate.

\textit{Asynchronous} \textsc{GoRank}, described in \cref{alg:async-gorank}, operates similarly to its \textit{synchronous} counterpart by exploiting the fact that ranks can be estimated via pairwise comparisons. Specifically, for $k = 1, \ldots, n$,
\begin{equation}\label{eq:ranks}
r_k = 1 + n\left( \frac{1}{n}\sum_{l=1}^n \mathbb{I}_{\{X_k > X_l\}}\right)=1+nr'_k.
\end{equation}
The algorithm proceeds as follows: each node $k \in [n]$ maintains local estimates $R_k(t)$ and $R'_k(t)$ of the rank $r_k$ and its normalized form $r'_k$, respectively, at iteration $t \geq 1$. Additionally, each node keeps a local iteration counter $C_k(t)$ and an auxiliary observation $Y_k(t)$. When nodes $i$ and $j$ are selected, they exchange their auxiliary observations. Then,  each node $k \in \{i,j\}$ increments its counter $C_k(t)$ and updates its estimate $R'_k(t)$ by computing the running average of the previous estimate $R'_k(t-1)$ and the indicator function $\mathbb{I}_{\{X_k > Y_k(t-1)\}}$. Finally, $R_k(t)$ is updated using \cref{eq:ranks}.

\begin{algorithm}[htbp]
        \caption{Asynchronous GoRank}
        \label{alg:async-gorank}
        \begin{algorithmic}[1]
        \STATE \textbf{Init:} For each \(k\in[n]\), \(Y_k \gets X_k\), \(R'_k \gets 0\), $C_k \gets 1$.  
        \FOR{\(t=0, 1, \ldots\)}
        \STATE Draw \(e=(i, j) \in E\) with probability $p_e>0$.
        \FOR{\(k\in \{i, j\}\)}
        \STATE Set $R'_k \leftarrow\left(1-1/C_k\right) R'_k +(1/C_k)\mathbb{I}_{\{X_k > Y_k\}}$.
        \STATE Update rank estimate: \(R_k \leftarrow n R'_k + 1\).
        \STATE Set $C_k \gets C_k + 1$.
        \ENDFOR
        \STATE Swap auxiliary observation: \(Y_i \leftrightarrow Y_j\). 
        \ENDFOR
        \STATE \textbf{Output:} Estimate of ranks \(R_k\). 
        \end{algorithmic}
\end{algorithm}

This random swapping procedure induces a random walk for each observation on the graph, characterized by the permutation matrix $\mathbf{W}_1(t) = \mathbf{I}_n - (\mathbf{e}_i - \mathbf{e}_j)(\mathbf{e}_i - \mathbf{e}_j)^\top$. Taking the expectation over the edge sampling process yields $\mathbf{W}_1 := \mathbb{E}[\mathbf{W}_1(t)] = \mathbf{I}_n - (1/|E|) \mathbf{L}$. For additional details, see \cite{van2025robust}. For each $k \in [n]$, define $\mathbf{h}_k = (\mathbb{I}_{\{X_k > X_1\}}, \ldots, \mathbb{I}_{\{X_k > X_n\}})^\top$; the true rank of observation $X_k$ is then given by $r_k = \mathbf{h}_k^\top \mathbf{1}_n + 1$. Let $p_k$ denote the probability that node $k$ is selected at each iteration. Under uniform node sampling, note that $p_k \geq 1/n$.  At iteration $t=0$, the observations have not been swapped yet, we obtain $R_k'(0) = 0$ and $R_k(0)=1$. For iteration $t\geq1$, the rank estimates can be expressed as
$$
R_k(t) = 1 + n \cdot \left( \frac{1}{C_k(t)} \sum_{s=1}^t \delta_k(s)\, \mathbf{h}_k^\top \mathbf{W}_1(s{-}1{:})^\top \mathbf{e}_k \right),
$$
where, for $s \geq 1$, the variables $\delta_k(s) \sim \mathcal{B}(p_k)$ are i.i.d. Bernoulli random variables with parameter $p_k$, and $\mathbf{W}_1(s{:}) = \mathbf{W}_1(s) \cdots \mathbf{W}_1(0)$. The normalization factor is given by $C_k(t) = \sum_{s=1}^t \delta_k(s)$. 

\begin{remark}
Real-world scenarios often contain tied values. When the dataset has many ties, \textsc{GoRank} (which assumes no ties) exhibits a clear bias. The algorithm can be extended to handle the case of $\ell \geq 2$ tied observations using the mid-rank method: the rank assigned to the $\ell$ tied values is the average of the ranks they would have received in the absence of ties, that is, the average of $p+1, p+2, \ldots, p+\ell$, which equals $p + {(\ell + 1)/2}$; see \cite{kendall1945treatment}. In the presence of ties, the rank formula can be rewritten as 
\[
r_k = \frac{1}{2} + n r'_k, \quad
r'_k = \frac{1}{n} \sum_{l=1}^n \left(\mathbb{I}_{\{X_k > X_l\}} + \frac{1}{2} \mathbb{I}_{\{X_k = X_l\}}\right).
\]
This is again an average of indicator functions, and the algorithm update can be modified by replacing $\mathbb{I}_{\{X_k > Y_k\}}$ with $\mathbb{I}_{\{X_k > Y_k\}} + \frac{1}{2} \mathbb{I}_{\{X_k = Y_k\}}$. Note that in our algorithm, the case $\mathbb{I}_{\{X_k = X_k\}}$ occurs naturally since we use auxiliary variables. 
\end{remark}

We now present a convergence result for the \textit{Asynchronous} \textsc{GoRank} algorithm, which establishes that the expected absolute error decreases at a rate of \( \mathcal{O}(1/\smash{\sqrt{ct}}) \). 

\begin{theorem}[Convergence of Asynchronous \textsc{GoRank}]
\label{thm:async-gorank}
Let $R_k(t)$ denote the estimate in \cref{alg:async-gorank}. Then, we have, for any $t>0$:
$$
\mathbb{E} \left[ |  R_k(t) - r_k|\right]  \leq \mathcal{O}\left(\frac{1}{\sqrt{ct}}\right) \cdot \sigma_n(r_k), 
$$
where \( c\) represents the connectivity of the graph and corresponds to the spectral gap (or second smallest eigenvalue) of the Laplacian defined as $\mathbf L(P) =\sum_{e \in E} p_e\mathbf{L}_e$, and the rank functional \( \sigma_n(r_k)=n^{3/2}\cdot\phi((r_k-1)/n) \) is determined by the score generating function $\phi:u\in (0,1)\to \sqrt{u(1-u)}$.
\end{theorem}

The remainder of this section is dedicated to proving \cref{thm:async-gorank}. To this end, it is equivalent to study the convergence of $R'_k(t)$ to $\bar{h}_k = (1/n) \mathbf h_k^\top \mathbf{1}_n$. Observe that
$$
\mathbb{E}[|R'_k(t) - \bar{h}_k|] = \mathbb{E}[|\tilde{S}_k(t)|/C_k(t)],
$$
where $\tilde{S}_k(t) = \sum_{s=1}^t \delta_k(s)\, \mathbf h_k^\top \tilde{\mathbf{W}}_1(s{-}1{:})^\top\, \mathbf e_k,$ and $\tilde{\mathbf{W}}_1(s{-}1{:}) = \mathbf{W}_1(s{-}1{:}) - (1/n) \mathbf{1}_n \mathbf{1}_n^\top$. Thus, the problem reduces to analyzing the expectation of a ratio of random variables. To deal with the non-linear term $1/C_k(t)$, we use the Taylor expansion of the function $f(x) = 1/x$ around $x = a$, where
$$
\frac{1}{x} = \frac{1}{a} - \frac{x - a}{a^2} + \frac{(x - a)^2}{x a^2},
$$
with $a = \mathbb{E}[C_k(t)] = p_k t$ and $x=C_k(t)$. Multiplying by $\tilde{S}_k(t)$, we obtain
\begin{align}
\label{eq:term1}
&\mathbb{E}[|{\tilde{S}_k(t)/C_k(t)}|]
\leq (1/p_kt)\mathbb{E}[|{\tilde{S}_k(t)}|] \\
\label{eq:term2}
&+ (1/p_k^2t^2)\mathbb{E}[|{(C_k(t) - p_k t)\tilde{S}_k(t)}|] \\ 
\label{eq:term3}
&+ (1/p_k^2t^2)\mathbb{E}[{(C_k(t) - p_k t)^2 |\tilde{S}_k(t)}|/C_k(t)].
\end{align}
We now bound the terms in \cref{eq:term1,eq:term2,eq:term3} via the following lemmas; proofs are deferred to the technical appendix.

\begin{lemma}
\label{lem:term1} 
We bound the term in \cref{eq:term1} as follows:
\[
\frac{1}{p_kt}\mathbb{E}[|\tilde{S}_k(t)|] \leq \sqrt{\frac{2n}{ct}} \cdot \|\mathbf{h}_k - \bar{h}_k\mathbf{1}_n \|.
\]
\end{lemma}

\begin{lemma}
\label{lem:term2} 
We bound the term in \cref{eq:term2} as follows:
\[
\frac{1}{(p_kt)^2}\mathbb{E}[|{(C_k(t) - p_k t)\tilde{S}_k(t)}|]  \leq \frac{\sqrt{2}n}{ct} \cdot \| \mathbf{h}_k - \bar{h}_k\mathbf{1}_n \|.
\]
\end{lemma}

\begin{lemma}
\label{lem:term3} 
There exists a constant $\kappa>0$ such that we can bound the term in \cref{eq:term3} as follows:
\[
\frac{1}{(p_kt)^2}\mathbb{E}[{(C_k(t) - p_k t)^2 |\tilde{S}_k(t)}|/C_k(t)]  \leq  \frac{\kappa}{\sqrt{ct}}\cdot \| \mathbf h_k - \bar{h}_k\|.
\]
\end{lemma}

Combining all previous lemmas, we obtain
\begin{equation*}
|\mathbb{E}[{\tilde{S}_k(t)/C_k(t)}]| \leq \mathcal{O}\left(\frac{1}{\sqrt{ct}}\right) \cdot\| \mathbf{\tilde h}_k \|.
\end{equation*}
Observe that $n \| \mathbf{\tilde h}_k \| = \sigma_n(r_k)$, see \cite{van2025robust}. Finally, substituting \( \mathbb{E}[R_k(t)] = n\mathbb{E}[R'_k(t)]+1\) along with \( r_k = n\bar h_k +1 \), completes the proof. Note that the same convergence guarantees for \textsc{GoRank} with ties hold when replacing $\mathbf{h}_k$ with $\mathbf{v}_k = (v_{k1}, \ldots, v_{kn})^\top$, where $v_{kl} = \mathbb{I}_{\{X_k > X_l\}} + (1/2) \mathbb{I}_{\{X_k = X_l\}}$.

\section{Gossip Estimation of Rank-based Statistics}

In this section, we introduce an asynchronous gossip algorithm for estimating rank-based statistics (see \cref{alg:general-gotrim}), and apply it to the estimation of the Wilcoxon statistic. The proposed algorithm generalizes the \textsc{GoTrim} algorithm by allowing for arbitrary rank-based weight functions $f$ and transformation functions $g$ applied to the observations. 

Let $Z_k(t)$ and $W_k(t)$ denote the local estimates of the statistic and weight, respectively, at node $k$ and iteration $t$. Recall from \cref{eq:rank-stat} that a rank-based statistic can be computed using standard gossip averaging:
$$
\frac{1}{n} \sum_{l=1}^n W_l(t) \cdot g(X_l),
$$
where $W_l(t) = n \cdot f(R_l(t))$, and $R_l(t)$ denotes the rank of node $l$ at time $t$. Since ranks $R_k(t)$ evolve over time, the algorithm dynamically adjusts to account for past estimation errors. At each iteration, it corrects the estimate by injecting the difference term $\left(W_k(t) - W_k(t-1)\right) \cdot g(X_k)$ into the averaging process. This update consists of two key components:
(1) An averaging step, described by the matrix $\mathbf{W}_2(t) = \mathbf{I}_n - (1/2)(e_i - e_j)(e_i - e_j)^\top$; (2) A correction step based on rank updates. 

\begin{algorithm}[htbp]
\caption{Gossip Estimation of Rank-based Statistics}
\label{alg:general-gotrim}
\begin{algorithmic}[1]
\STATE \textbf{Input:} Choice of ranking algorithm \texttt{rank} (\eg Asynchronous GoRank).
\STATE \textbf{Init:}  \(\forall k\), \(Z_k \leftarrow 0\), \(R_k \leftarrow \texttt{rank}.\texttt{init}(k)\), \(W_k \leftarrow 0\).
\FOR{\(t=0, 1, \ldots\)}
\STATE Draw \(e=(i, j) \in E\) with probability $p_e>0$.
\FOR{\(k\in \{i,j\}\)}
\STATE Update rank: \(R_k \leftarrow \texttt{rank.update}(k)\).
\STATE Update weight: $W_k' \gets n \cdot f(R_k)$
\STATE Set \(Z_k \gets Z_k + (W_k' - W_k)\cdot g(X_k)\).
\STATE Set $W_k \gets W'_k$
\ENDFOR
\STATE Set \(Z_i, Z_j \leftarrow (Z_i + Z_j)/2\).
\STATE Swap auxiliary variables: \texttt{swap(i, j)}.
\ENDFOR
\STATE \textbf{Output:} Estimate of statistic \(Z_k\).
\end{algorithmic}
\end{algorithm}

Putting these together, the evolution of the estimate vector $\mathbf{Z}(t) = (Z_1(t), \ldots, Z_n(t))$ can be written as: for $t > 0$, $\mathbf{Z}(t) = \mathbf{W}_2(t) \left(\mathbf{Z}(t-1) + \boldsymbol{\delta}(t) \odot \boldsymbol{\tilde{X}}\right)$, where $\boldsymbol{\delta}(t) = (\delta_1(t), \ldots, \delta_n(t))$ with $\delta_k(t) = W_k(t) - W_k(t-1)$ and $\boldsymbol{\tilde{X}} = (g(X_1), \ldots, g(X_n))$. Taking the expectation over the edge sampling process yields:
$$
\mathbb{E}[\mathbf{Z}(t)] = \mathbf{W}_2 \left(\mathbb{E}[\mathbf{Z}(t-1)] + \Delta \boldsymbol{w}(t) \odot \boldsymbol{\tilde X}\right),
$$
where $\mathbf{W}_2 = \mathbf{I}_n - (1/2|E|) \mathbf{L}$ and $\Delta \boldsymbol{w}(t) = \mathbb{E}[\boldsymbol{\delta}(t)]$. After the first iteration, since $Z_k(0) = 0$, the expected estimate becomes $\mathbb{E}[\boldsymbol{Z}(1)] = \mathbf{W}_2 \Delta \boldsymbol{w}(1) \odot \boldsymbol{\tilde X}$. By recursion, for any $t > 0$, the expected value of the estimates evolves as:
$$
\mathbb{E}[\boldsymbol{Z}(t)] = \sum_{s=1}^{t} \mathbf{W}_2^{t+1-s} \Delta \boldsymbol{w}(s) \odot \boldsymbol{\tilde X}.
$$
Note that, similar to the convergence analysis of \textsc{GoTrim}, this gossip algorithm will converge to the correct value, since the averaging matrix can be decomposed as \(\mathbf{W}_2 = (1/n) \mathbf{1}_n \mathbf{1}_n^\top + \tilde{\mathbf{W}}_2, \) with spectral gap satisfying $0 < c_2 < 1$, where $c_2 = c/2$. Indeed, denoting $\mathbf{S}(s) = \Delta \boldsymbol{w}(s) \odot \tilde{\boldsymbol{X}}$, we have
\begin{equation*}
    \sum_{s=1}^{t}\frac{1}{n} \mathbf{1}_n \mathbf{1}_n^\top \mathbf S(s)  = \left(\frac{1}{n} \sum_{k=1}^n  n \cdot {\mathbb{E}[f(R_k(t))]} \cdot g(X_k) \right)\mathbf{1}_n.
\end{equation*}
However, the full convergence analysis requires specific assumptions on the functions $f$ and $g$. In the following, we address the special case of the the Wilcoxon statistic, see \cref{eq:wilcoxon}. Following the convergence analysis of \textsc{GoTrim} \cite[Theorem~3]{van2025robust}, we establish that $Z_k(t)$ converges in expectation to the target Wilcoxon statistic at rate $\mathcal{O}(1/{c^2t})$.

\begin{theorem}[Convergence of Wilcoxon Statistic Estimation]
\label{thm:wilcoxon}
Let $t_n$ be defined as in \cref{eq:wilcoxon}, and let $\boldsymbol{Z}(t)$ denote the output of \cref{alg:general-gotrim} with $f = \operatorname{id}$ and $g(X_k) = \mathbb{I}_{\{X_k \in S_1\}}$. Then, there exists $T^*$, such that for any \( t > T^* \), we have
\[
\left\| \mathbb{E}[\boldsymbol{Z}(t)] - t_n \mathbf{1}_n \right\| 
\leq \mathcal{O}\left(\frac{1}{{c^2 t}} \right),
\]
where $c$ is defined in \cref{thm:async-gorank}.
\end{theorem}

The result follows by adapting the proof of Theorem 3 in \cite{van2025robust} and is deferred to the technical appendix

\begin{remark}
The Wilcoxon statistic can be used to perform a two-sample test, \textit{i.e.,} to test the null hypothesis that two samples are identically distributed. Intuitively, if the $X_j$ form a sample from a stochastically larger distribution, then the ranks of the $X_j$ in the pooled sample should be relatively large. The Wilcoxon rank sum test proceeds as follows: (1) compute the mean $\mu={n_1(n+1)}/2$ and standard deviation $\sigma=\sqrt{{n_1 n_2(n+1)}/12}$; (2) compute the z-score $z = (t_n-\mu)/\sigma$ and derive the $p$-value using normal distribution approximation: $p=2(1-\Phi(|z|))$ where $\Phi$ denotes the cumulative distribution function of the standard normal distribution.     
\end{remark}

\section{Faster Gossip for Trimmed Means}

In this section, we introduce \textit{Adaptive} \textsc{GoTrim}, an improved version of \textsc{GoTrim} designed to accelerate convergence when the trimming parameter $\alpha$ is large (\textit{i.e.,} close to $0.5$). The trimmed mean is denoted $\bar{x}_{\alpha}$. The algorithm is outlined in \cref{alg:gotrim-bias}. Our prior work established that the expected estimates produced by \textsc{GoTrim} converge to the true trimmed mean at a rate of $\mathcal{O}(1/ct)$. Here, we analyze the convergence of the expected absolute error, and show that it satisfies a bound of $\mathcal{O}(\sqrt{1/ct})$ in the synchronous setting. We then extend this result to derive the same convergence rate.

\begin{algorithm}[htbp]
\caption{Fast Asynchronous GoTrim}
\label{alg:gotrim-bias}
\begin{algorithmic}[1]
\STATE \textbf{Input:} Trimming level $\alpha\in (0,1/2)$, and choice of ranking algorithm \texttt{rank} (\eg Asynchronous GoRank).
\STATE \textbf{Init:} \(\forall k\), \(Z_k \leftarrow 0\), \(W_k \leftarrow 0\) and \(R_k \leftarrow \texttt{rank}.\texttt{init}(k)\).
\FOR{\(t=1, 2, \ldots\)}
\STATE Draw \(e=(i, j) \in E\) with probability $p_e > 0$.
\FOR{\(k\in \{i,j\}\)}
\STATE Update rank: \(R_k \leftarrow \texttt{rank.update}(k, t)\).
\STATE Set \(W^{\prime}_k \leftarrow w_{n, \alpha}(R_k)\).
\STATE Set \(N_k \leftarrow N_k + (W^{\prime}_k - W_k) \cdot X_k\).
\STATE Set \(M_k \leftarrow M_k + (W^{\prime}_k - W_k)\)
\STATE Set \(W_k \leftarrow W^{\prime}_k\).
\ENDFOR
\STATE Set \(N_i, N_j \leftarrow (N_i + N_j)/2\).
\STATE Set \(M_i, M_j \leftarrow (M_i + M_j)/2\).
\STATE Swap auxiliary variables: \texttt{swap(i, j)}
\ENDFOR
\STATE \textbf{Output:} Estimate of trimmed mean \(N_k/\operatorname{max}(1, M_k)\).
\end{algorithmic}
\end{algorithm}

Unlike the original \textsc{GoTrim}, the proposed algorithm maintains an estimate of the average of the weights, denoted by $M_k$. Instead of returning the raw estimate $N_k$ directly, the algorithm normalizes $N_k$ by dividing it by $\max(1, M_k)$. In theory, this normalization should not significantly affect the outcome since $M_k$ converges to 1 over time. However, in practice, we observe that this adjustment helps mitigate the initial bias present in \textsc{GoTrim}. Early in the process, when the rank estimates are less precise, the weighted average can be improperly normalized. By scaling $N_k$ with $\max(1, M_k)$, the algorithm corrects for this bias. 

We now turn to the convergence guarantees. While bounding the expected value of the estimates is relatively straightforward, obtaining a bound on the expected absolute error requires a more refined convergence analysis.

Denote $\boldsymbol{Z}(t) = \boldsymbol{N}(t) / \boldsymbol{D}(t)$ with element-wise division, where $ \boldsymbol{N}(t) = \sum_{s=1}^{t} \boldsymbol{W}_2(t{:}s) \, \Delta \boldsymbol{W}(s) \odot \boldsymbol{X},$
with $\Delta \boldsymbol{W}(s) = \boldsymbol{W}(s) - \boldsymbol{W}(s{-}1)$, and $\boldsymbol{D}(t) = \max(1, \boldsymbol{M}(t))$ defined element-wise, where
$ \boldsymbol{M}(t) = \sum_{s=1}^{t} \boldsymbol{W}_2(t{:}s) \, \Delta \boldsymbol{W}(s) \odot \boldsymbol{1}_n.$

We begin by establishing a convergence guarantee for $\boldsymbol{N}(t)$, which corresponds to the original \textsc{GoTrim} estimate in the synchronous setting. 

\begin{theorem}
\label{thm:original-gotrim}
Let $\boldsymbol{N}(t)$ denote the original \textsc{GoTrim} estimate. Then, there exists $T^*>0$, such that for any \( t > T^* \), we have
$$ \mathbb{E}\left\|\mathbf{N}(t) - \bar{x}_{\alpha} \mathbf{1}_n \right\|_2 \leq \mathcal{O}\left( \frac{1}{\tilde c \sqrt{ct}}\right)\cdot \| {\mathbf K}\odot \mathbf X\|$$where $ \mathbf K =  \left\{\frac{\sqrt{6}\sigma_n(r_i)}{\gamma_i(1 - 2\alpha)}\right\}_{i=1}^n$ , $\tilde c = 1-\tilde \lambda$ with $\tilde \lambda = \sqrt{\lambda_2(2)}$, and $c$ is as defined in \cref{thm:wilcoxon}.
\end{theorem}

We first present a useful lemma on the convergence of the weights. This lemma extends Lemma 1 in \cite{van2025robust} and provides a bound on the expected absolute error rather than on the error of the expected estimate.  

\begin{lemma}
\label{lem:conv-weight}
For all \( k \in [n] \) and \( t > 0 \), we have
\begin{equation}
\mathbb{E}\left|W_k(t) - w_{n,\alpha}(r_k)\right| \leq \frac{3}{ct}\cdot \frac{\sigma_n(r_k)^2}{\gamma_k^2(1-2\alpha) },
\end{equation}
where \( \gamma_k = \min\left(\left|r_k - a\right|, \left|r_k - b\right|\right) \geq \frac{1}{2} \), with \( a = \lfloor \alpha n \rfloor + \frac{1}{2} \) and \( b = n - \lfloor \alpha n \rfloor + \frac{1}{2} \) denoting the endpoints of the interval \( I_{n,\alpha} \). The constants \( c \) and \( \sigma_n(\cdot) \) are as defined in \cref{thm:original-gotrim}.
\end{lemma}

Using the previous lemma, we can establish the convergence of the expected absolute error of the original \textsc{GoTrim} estimates. The proofs of \cref{lem:conv-weight} and \cref{thm:original-gotrim} are deferred to the technical appendix.

Now, we can prove that \textit{adaptive} \textsc{GoTrim} will have converge of the same order, stated in the following Theorem. 

\begin{theorem}
\label{thm:adaptive-gotrim}
Let $\boldsymbol{Z}(t)$ denote the \textit{adaptive} \textsc{GoTrim} estimate. Then, there exists $T^*>0$, such that for any \( t > T^* \), we have
$$\mathbb{E}\left\|\boldsymbol{Z}(t) - \bar{x}_{\alpha} \boldsymbol{1}_n \right\|_2 \leq \mathcal{O}\left( \frac{1}{\tilde c \sqrt{ct}}\right)\cdot \| {\mathbf K}\odot (\mathbf X + |\bar{x}_{\alpha}|\mathbf 1_n)\| $$
where $\mathbf{K}$, $\tilde c$, and $c$ are as defined in \cref{thm:adaptive-gotrim}. 
\end{theorem}
The proof is deferred to the technical appendix.
\section{Numerical Experiments}

\begin{figure*}[htbp]
    \centering
    \begin{subfigure}[b]{0.32\textwidth}
        \centering
        \includegraphics[height=5cm, width=\textwidth]{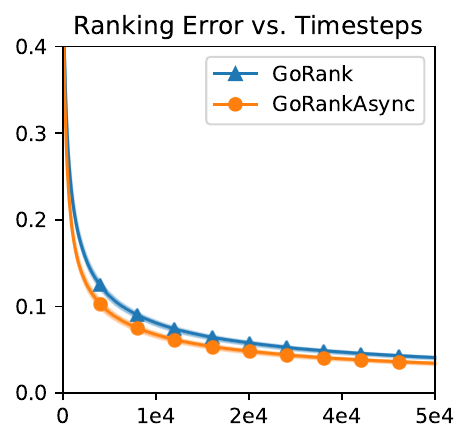}
        \caption{GoRank vs. GoRankAsync}
        \label{fig:sub1}
    \end{subfigure}
    \hfill
    \begin{subfigure}[b]{0.33\textwidth}
        \centering
        \includegraphics[height=5cm, width=1\textwidth]{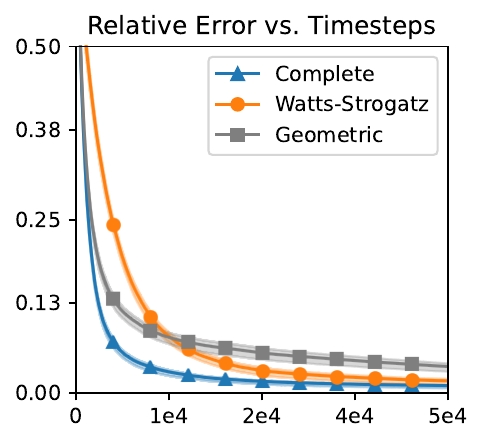}
        \caption{Wilcoxon Statistic Estimation}
        \label{fig:sub2}
    \end{subfigure}
    \hfill
    \begin{subfigure}[b]{0.33\textwidth}
        \centering
        \includegraphics[height=5cm, width=1\textwidth]{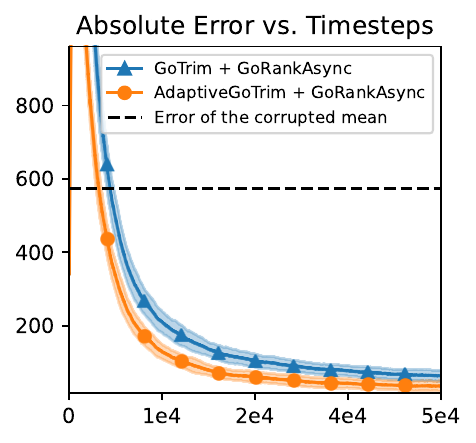}
        \caption{GoTrim with Bias Correction}
        \label{fig:sub3}
    \end{subfigure}
    \caption{Figure (a) shows that \textit{Asynchronous} \textsc{GoRank} converges slightly faster than its synchronous version. Figure (b) shows fast convergence for the Wilcoxon estimation even on sparse graphs, with convergence speed matching the theoretical bound. Figure (c) shows that \textit{adaptive} \textsc{GoTrim} slightly outperforms its original version (both improving over the naive mean).}
\end{figure*}

\paragraph{Setup} For \cref{alg:async-gorank} and \cref{alg:gotrim-bias}, we conduct experiments on a dataset \( S = \{1, \dots, n\} \) with \( n = 500 \), distributed across nodes of a communication graph. We evaluate \cref{alg:async-gorank} using the normalized absolute error between estimated and true ranks; that is, for node \( k \) at iteration \( t \), the error is defined as \( \ell_k(t) = |R_k(t) - r_k| / n \). Figure~(a) compares the performance of synchronous (see \cite{van2025robust}) and asynchronous GoRank on a Watts--Strogatz graph (\( c = 3.11 \times 10^{-4} \)).

\cref{alg:general-gotrim} is evaluated on the estimation of the Wilcoxon statistic. For this, we consider a dataset \( S_1 \) of \( n_1 = 250 \) samples from a Cauchy distribution with location $0.8$ and scale $1.0$, and a second dataset \( S_2 \) of \( n_2 = 250 \) samples from a Cauchy distribution with location $0.0$ and scale $1.0$. We compute the average relative error defined by $\ell(t) = \frac{1}{n} \sum_k \frac{|Z_k(t) - t_n|}{t_n}$.
Figure~(c) compares the convergence of \cref{alg:gotrim-bias} on three graph topologies: the complete graph (\( c = 4.01 \times 10^{-3} \)), in which every node is directly connected to all others, yielding maximal connectivity; a random connected geometric graph (\( c = 3.57 \times 10^{-5} \)) with radius \( \rho = 0.1 \); and the Watts--Strogatz network (\( c = 3.11 \times 10^{-4} \)), a randomized graph with average degree \( k = 4 \) and rewiring probability \( p = 0.2 \), offering intermediate connectivity between the complete and geometric graphs.

\cref{alg:gotrim-bias} is evaluated on a corrupted dataset. Here, the dataset \( S \) is contaminated by replacing a fraction \( \lfloor \varepsilon n \rfloor \) with \( \varepsilon = 0.3 \) of the values via scaling: a value \( x \) is changed to \( sx \) with \( s = 10 \). While this is a relatively simple form of corruption, it is sufficient to \textit{break down} the classical mean. We measure performance using the absolute error between the estimated and true trimmed mean. For node \( k \) at iteration \( t \), the error is given by $\ell(t) = \frac{1}{n} \sum_k \left| Z_k(t) - \bar{x}_{\alpha} \right|,$ where $\bar{x}_{\alpha}$ represents the true trimmed mean with \( \alpha = 0.4 \). Figure~(b) compares the performance of GoTrim with and without the bias correction.

All algorithms are run for \(5 \times 10^{4} \) iterations, averaged over $100$ trials. For each trial, the data is randomly distributed across the graph. The plots display the mean error and standard deviation over the iterations. The code for our experiments will be publicly available.

\paragraph{Results} Figure~(a) shows that the asynchronous variant actually converges faster than synchronous GoRank. While this is not captured in our convergence bounds from \cref{thm:async-gorank}, these bounds still guarantee that both synchronous and asynchronous versions share the same order of convergence rates, which is consistent with our empirical observations.  Figure~(b) demonstrates that our algorithm can efficiently estimate the Wilcoxon statistic even in sparsely connected graphs. Furthermore, the asymptotic behavior aligns with the convergence bounds in \cref{thm:wilcoxon}: the complete graph achieves the fastest convergence, while the geometric graph is the slowest. Figure~(c) shows that GoTrim with bias correction slightly outperforms the vanilla GoTrim, as expected. Both variants quickly improve over the naive (corrupted) mean, but the bias-corrected version yields better overall accuracy.

\section{Conclusion}

We presented new gossip algorithms for estimating rank-based statistics, including trimmed means and the Wilcoxon statistic, enabling robust decentralized estimation and hypothesis testing. Our methods operate in the asynchronous setting and are supported by experiments across diverse data distributions and network topologies. We established novel theoretical guarantees, with our asynchronous result for rank estimation being, to the best of our knowledge, the first of its kind. Note that our theoretical results extend to arbitrary edge sampling probabilities, provided that the graph remains connected. Future work will explore adaptive sampling strategies to further enhance convergence speed and robustness.

\section{Technical Appendix}

In this section, we first present key properties of the transition matrices, essential for our theoretical results. We then provide the proofs of lemmas that were previously omitted.

\subsection{Preliminary Results}

\begin{lemma}
\label{lem:properties}
Assume the graph $\mathcal{G }=(V, E)$ is connected and non-bipartite. Let \( t > 0 \).  If at iteration \( t \), edge \(e= (i, j) \) is selected with probability \( p_e > 0\), then the transition matrix is given by $ \mathbf{W}_{\alpha}(t) = \mathbf{I}_n -\mathbf{L}_e/\alpha $ where the elementary Laplacian is defined as $\mathbf{L}_e  = \left( \mathbf{e}_i - \mathbf{e}_j \right) \left( \mathbf{e}_i - \mathbf{e}_j \right)^{\top}$ and where $\alpha=1$ corresponds to the swapping matrix and $\alpha=2$ to the averaging matrix.
The following properties hold:
\begin{enumerate}[label=(\alph*)]
    \item The matrices are symmetric and doubly stochastic, meaning that $\mathbf{W}_\alpha(t) \mathbf{1}_n = \mathbf{1}_n$ and  $\mathbf{1}_n^\top \mathbf{W}_\alpha(t) = \mathbf{1}_n^\top$.
They satisfy:  $\mathbf{W}_2(t)^2 = \mathbf{W}_2(t)$ and $\quad \mathbf{W}_1(t)^2 = \mathbf{I}_n$.
    \item For \(\alpha \in \{1, 2\}\), we have, since  $\sum_{e \in E} p_e = 1$,  
\begin{equation}
\mathbf{W}_\alpha = \mathbb{E}[\mathbf{W}_\alpha(t)] = \mathbf{I}_n - \frac{1}{\alpha}\sum_{e \in E} p_e\mathbf{L}_e.
\end{equation}
The matrix \(\mathbf{W}_\alpha\) is also doubly stochastic and it follows that \(\mathbf{1}_n\) is an eigenvector with eigenvalue 1. 

\item The matrix $\tilde{\mathbf{W}}_\alpha \triangleq \mathbf{W}_\alpha -{\mathbf{1}_n \mathbf{1}_n^\top}/n$ satisfies, by construction, \(\tilde{\mathbf{W}} \mathbf{1}_n = 0\), and it can be shown that $\| \tilde{\mathbf{W}}_\alpha \|_{\operatorname{op}} \leq \lambda_2(\alpha)$, where \(\lambda_2(\alpha)\) is the second largest eigenvalue of \(\mathbf{W}_\alpha\) and $\|\cdot\|_{\operatorname{op}}$ denotes the operator norm of a matrix. Moreover, the eigenvalue \(\lambda_2(\alpha)\) satisfies $0 \leq \lambda_2(\alpha) < 1$ and $\lambda_2(\alpha) = 1 - {c/\alpha}$ where $c$ the spectral gap (or second smallest eigenvalue) of the Laplacian of the weighted graph defined as $\mathbf L(P) =\sum_{e \in E} p_e\mathbf{L}_e$ . 

\end{enumerate}    
\end{lemma}

\begin{proof}
These results are analogous to those in \cite{van2025robust}, with the difference that the uniform edge sampling \( p_e = 1/|E| \) is replaced by an arbitrary distribution. The proofs remain valid under the assumption that \( p_e > 0 \). For more details, we refer the reader to \cite{boyd2006randomized, colin2015extending}.
\end{proof}

\subsection{Proof of \cref{lem:term1}}

\begin{proof}
Applying Cauchy--Schwarz inequality, we obtain
$\mathbb{E}[|\tilde{S}_k(t)|] \leq \sqrt{\mathbb{E}[\tilde{S}_k(t)^2]}.$ Define the centered vector $\tilde{\mathbf{h}}_k := \mathbf{h}_k - \bar{h}_k \mathbf{1}_n$. Using the fact that $\tilde{\mathbf{W}}_1(s{-}1{:}) \mathbf{1}_n = 0$, 
\[
\tilde{S}_k(t) = \sum_{s=1}^t \delta_k(s)\, \tilde{\mathbf{h}}_k^\top \tilde{\mathbf{W}}_1(s{-}1{:})^\top \mathbf{e}_k.
\]The corresponding squared term can then be written as:
\begin{align*}
  &\tilde{S}_k(t)^2 = \sum_{s = 1}^{t} \delta_k(s) \mathbf{\tilde h}_k^\top \tilde{\mathbf{W}}_1(s{-}1{:})^\top  \mathbf{e}_k \mathbf{e}_k^\top \tilde{\mathbf{W}}_1(s{-}1{:}) \mathbf{\tilde h}_k  \\
  &+ 2\sum_{s <u }\delta_k(s) \delta_k(u) \mathbf{\tilde h}_k^\top \tilde{\mathbf{W}}_1(s{-}1{:})^\top  \mathbf{e}_k \mathbf{e}_k^\top \tilde{\mathbf{W}}_1(u{-}1{:}) \mathbf{\tilde h}_k.
\end{align*}
We define \(v(s) := \delta_k(s)\tilde{\mathbf{h}}_k^\top \tilde{\mathbf{W}}_1(s{-}1{:})^\top \mathbf{e}_k  \) and consider the term  \(  \sum_{s=1}^{t} v(s)^2 \). Let $\mathbf{J} = (1/n) \mathbf{1}_n \mathbf{1}_n^\top$. Since we have $\tilde{\mathbf{h}}_k^\top \mathbf{J} = 0$, $\|\mathbf{e}_k \mathbf{e}_k^\top \|_{\operatorname{op}} \leq 1$, and the permutation matrix is its own inverse, we can bound $v(s)^2$ as follows:
\[
\delta_k(s)\tilde{\mathbf{h}}_k^\top \mathbf{W}_1(s{-}1{:})^\top \mathbf{e}_k \mathbf{e}_k^\top \mathbf{W}_1(s{-}1{:}) \tilde{\mathbf{h}}_k \leq \delta_k(s)\tilde{\mathbf{h}}_k^\top  \tilde{\mathbf{h}}_k.
\]Taking expectation and using that $\delta_k(s)$ is a Bernoulli variable with mean $p_k$, we obtain
\[
\mathbb{E} \left[ \sum_{s=1}^{t} v(s)^2 \right] \leq p_k t\, \|\tilde{\mathbf{h}}_k\|^2.
\]
Next, we consider the second term: \( \sum_{s < u} v(s) v(u) \). Observe that, for \( s < u \), the product \( v(s) v(u) \) can be expressed as
\[
\delta_k(s) \delta_k(u)\tilde{\mathbf{h}}_k^\top {\mathbf{W}}_1(s{-}1{:})^\top \mathbf{e}_k \mathbf{e}_k^\top {\mathbf{W}}_1(u{-}1{:}s) 
{\mathbf{W}}_1(s{-}1{:}) \tilde{\mathbf{h}}_k .
\]
Define the vector \( \mathbf{y}(s) := \mathbf{W}_1(s{-}1{:}) \tilde{\mathbf{h}}_k \). Then, the conditional expectation of the cross term given \( \mathbf{W}_1(s{-}1{:}) \) is
\[
  p_k \, \mathbf y(s{-}1)^\top
\mathbf{e}_k \mathbf{e}_k^\top \mathbb{E}\left[ \delta_k(s){\mathbf{W}}_1(u{-}1{:}s)\right] \,  \mathbf y(s{-}1)\enspace.
\]
It remains to bound the expectation
$$ \mathbb{E}\left[ \delta_k(s)\tilde{\mathbf{W}}_1(u{-}1{:}s)\right] = \tilde{\mathbf{W}}_1^{u-s-1} \mathbb{E}\left[ \delta_k(s)\tilde{\mathbf{W}}_1(s)\right].$$
Let $L_{ij} = (\mathbf{e}_i - \mathbf{e}_j)(\mathbf{e}_i - \mathbf{e}_j)^\top$ denote the elementary Laplacian matrix. Observe that, the term $\mathbb{E}[\delta_k(s){\mathbf{W}}_1(s)]$ becomes
\begin{equation*}
\sum_{(i,j) \in E} p_{ij} \mathbb{I}_{\{k \in \{i,j\}\}} (\mathbf I_n - \mathbf L_{ij}) = p_k \left(\mathbf I_n - \frac{1}{p_k}\, \mathbf L_k\right),    
\end{equation*}
where $\mathbf L_k$ is the Laplacian of the weighted star graph formed by node $k$ and its neighbors $j$ with weight $p_{kj}$. Note that this graph is bipartite. Define $\mathbf P_k = \mathbf I_n - (1/p_k)\, \mathbf L_k$. Since the matrix $\mathbf P_k$ is also doubly stochastic, defining the projection $\tilde{\mathbf P}_k = \mathbf P_k - (1/n)\, \mathbf{1}_n \mathbf{1}_n^\top$, we have $\| \tilde {\mathbf P}_k \|_{\operatorname{op}}\leq1$. We can now bound the conditional expectation of the cross term:
$$\mathbb{E}[v(s) v(u) \mid \mathbf{W}_1(s{-}1{:})] \leq p_k^2 \, \lambda_2(1)^{u - s-1} \, \mathbf y(s{-}1)^\top \mathbf y(s{-}1),$$
where we used the assumptions $\|\tilde{\mathbf{W}}_1\|_{\operatorname{op}} \leq \lambda_2(1)$. Similar to the previous bound, we use the properties of the permutation matrices to obtain the following bound:
\[
\mathbb{E}[v(s) v(u)] \leq p_k^2 \,\lambda_2(1)^{u - s-1} \, \|\tilde{\mathbf{h}}_k\|^2 \enspace.
\]
To bound the sum over all such cross terms, observe that
\[
\sum_{s < u} \lambda_2(1)^{u - s-1} \leq \sum_{u=1}^{t} \sum_{d=0}^{u-2} \lambda_2(1)^d \leq \frac{t}{1 - \lambda_2(1)} \enspace.
\]
Setting $c = 1 - \lambda_2(1)$ and combining both terms, we have:
\[
\mathbb{E}[\tilde{S}_k(t)^2] \leq \left( p_k t + \frac{2p_k ^2t}{c}  \right) \|\tilde{\mathbf{h}}_k\|^2.
\]
Using \( c/p_k \leq n\), we obtain
\[
\frac{1}{(p_k t)^2} \mathbb{E}[\tilde{S}_k(t)^2] \leq \frac{2n}{c t} \cdot \|\tilde{\mathbf{h}}_k\|^2.
\]
Taking the square root finishes the proof.
\end{proof}

\subsection{Proof of \cref{lem:term2}}

\begin{proof}
By the Cauchy--Schwarz inequality, we have
$$\mathbb{E}[ |(C_k(t){-} p_k t)\tilde{S}_k(t)|] \leq \sqrt{\mathbb{E}\left[(C_k(t) {-} p_k t)^2\right]} \cdot  \sqrt{\mathbb{E}[\tilde{S}_k(t)^2]}.$$
The first term corresponds to the variance of the binomial random variable \( C_k(t) \sim \mathcal{B}(t, p_k) \), which satisfies 
\[
\mu_2 = \mathbb{E}\left[ (C_k(t) - p_k t)^2 \right] = t p_k (1 - p_k) \leq t p_k.
\]
The second term was previously bounded in the proof of Lemma 1: 
\[
\frac{1}{(p_k t)^2} \mathbb{E}[\tilde{S}_k(t)^2] \leq \frac{2n}{c t} \cdot \|\tilde{\mathbf{h}}_k\|^2.
\]
Denoting $m=|E|$ and using that $1/p_k \leq n$,  we obtain $\sqrt{\mu_2 /(p_kt)^2} \leq \sqrt{n/ct}$. Combining both bounds finishes the proof.  
\end{proof}

\subsection{Proof of \cref{lem:term3}}

\begin{proof}
By the Cauchy--Schwarz inequality and the fact that \( 1/C_k(t)^2 \leq 1 \), we obtain
$$\mathbb{E}\left[{\frac{(C_k(t){-} a)^2}{C_k(t)} |\tilde{S}_k(t)|}\right] \leq \sqrt{\mathbb{E}\left[(C_k(t) {-} a)^4 \right]} \cdot  \sqrt{\mathbb{E}[\tilde{S}_k(t)^2]}.$$
where we denote \( a := p_k t \) for convenience.
The first factor corresponds to the fourth central moment \( \mu_4 \) of the binomial distribution \( \mathcal{B}(t, p_k) \), which satisfies
\[
\mu_4 = t p_k (1 - p_k) \left[ 1 - 6 p_k (1 - p_k) + 3 t p_k (1 - p_k) \right].
\]
Using the bound \( \mu_4 \leq t p_k + 3 t^2 p_k^2 \), we deduce (in analogy with Lemma 2) that
\[
\frac{1}{(p_kt)^2}\mathbb{E}\left[(C_k(t) - p_k t)^4 \right] \leq  3 + \frac{1}{tp_k} \leq 3+\frac{n}{t}.
\]
Combining this with the bound from Lemma 1 on \( \mathbb{E}[\tilde{S}_k(t)^2] \), we obtain
\[
\frac{1}{(p_k t)^4} \mathbb{E}\left[(C_k(t) - p_k t)^4 \right]  \,\mathbb{E}[\tilde{S}_k(t)^2] \leq \left(3+\frac{n}{t}\right)\cdot\frac{2n}{c t} \cdot \|\tilde{\mathbf{h}}_k\|^2.
\]
Finally, choosing $\kappa > 0$ such that $\sqrt{6n + 2n^2/t} \leq \kappa$ completes the proof.
\end{proof}

\subsection{Proof of \cref{thm:wilcoxon} }

\begin{proof}
The result follows by adapting the proof of Theorem 3 in \cite{van2025robust}. Specifically, we replace $\mathbf{X}$ with $\mathbf{b} := (b_1, \ldots, b_n)$, where $\mathbf{b}(t)$ denotes the vector weights for the Wilcoxon statistic, and instead of the convergence bound for the weights, we use the analogous bound for the rank estimates: $\left| \mathbb{E}[R_k(t)] - r_k \right| \leq \mathcal{O}\left(1/{ct}\right).$ Indeed, denoting $\mathbf{Q}(t)$ the residuals, one has
\[
\left\|\mathbb{E}[\boldsymbol{Z}(t)] - t_n \mathbf{1}_n \right\| 
\leq \left| \sum_{k=1}^n (\mathbb{E}[R_k(t)] - r_k) b_k\right| \cdot \|\mathbf{1}_n\| + \|\mathbf{Q}(t)\|.
\]
The remainder of the argument proceeds identically to the proof of \cite[Theorem~3]{van2025robust}.
\end{proof}

\subsection{Proof of \cref{lem:conv-weight}}

\begin{proof}
We note that the argument used in the proof of Lemma~1 in \cite{van2025robust} can be directly applied to bound $\mathbb{E}| W_k(t) - w_{n,\alpha}(r_k)|$. Specifically, consider two cases: if \( r_k \in I_{n,\alpha} \), the expression simplifies to $\mathbb{E}\left|\mathbb{I}_{\{R_k(t) \notin I_{n,\alpha}\}}\right| = \mathbb{P}(R_k(t) \notin I_{n,\alpha})$; if \( r_k \notin I_{n,\alpha} \), the expression becomes $\mathbb{E}\left|\mathbb{I}_{\{R_k(t) \in I_{n,\alpha}\}}\right| = \mathbb{P}(R_k(t) \in I_{n,\alpha}).$ This aligns exactly with the setup of Lemma~1 in \cite{van2025robust}. Applying the same reasoning and steps from that proof yields the desired bound.   
\end{proof}

\subsection{Proof of \cref{thm:original-gotrim}}

\begin{proof}
Let $\mathbf{J} = (1/n) \mathbf{1}_n \mathbf{1}_n^\top$. Define $\bar{x}_{\alpha}(t) =(1/n) \cdot\left( \mathbf{W}(t) \odot \mathbf{X} \right)^\top \mathbf{1}_n.$ Note that $\sum_{s=1}^{t} \mathbf{J} \left( \Delta \mathbf{W}(s) \odot \mathbf{X} \right) = \bar{x}_{\alpha}(t) \mathbf{1}_n.$ Hence, denoting $\mathbf{S}(s) = \Delta \mathbf{W}(s) \odot \mathbf{X}$, we can write
\begin{align*}
\mathbb{E}\left\|\mathbf{N}(t) - \bar{x}_{\alpha} \mathbf{1}_n \right\|_2 
&\leq \mathbb{E}\|\sum_{s=1}^{t} \tilde{\mathbf{W}}_2(t{:}s) \, \mathbf{S}(s) \|_2 \\
&+ \mathbb{E}\left\| \left( \bar{x}_{\alpha}(t) - \bar{x}_{\alpha} \right) \mathbf{1}_n \right\|_2,    
\end{align*}
where $\tilde{\mathbf{W}_2}(t{:}s) = \mathbf{W}_2(t{:}s) - \mathbf{J}$. Denote $v(s) = \tilde{\mathbf{W}}_2(t{:}s) \, \mathbf{S}(s).$ Analyzing the first term amounts to analyzing $\mathbb{E}[\|v(s)\|_2^2]$, as
{
\begin{align*}
     \mathbb{E}\|\sum_{s=1}^{t} \tilde{\mathbf{W}}_2(t{:}s) \, \mathbf{S}(s) \|_2 
     &\leq \sum_{s=1}^{t} \mathbb{E}\| \tilde{\mathbf{W}}_2(t{:}s) \, \mathbf{S}(s) \|_2 \\
     &\leq \sum_{s=1}^{t} \sqrt{\mathbb{E}\| \tilde{\mathbf{W}}_2(t{:}s) \, \mathbf{S}(s) \|^2_2}.
\end{align*}
}
We now analyze \( \mathbb{E}[\|v(s)\|_2^2] = \mathbb{E}[v(s)^\top v(s)] \). Taking expectation conditional on \( \mathcal{F}_{t-1} \) (natural filtration), for \( s < t, \)
\begin{align*}
\mathbb{E}[v(s)^\top v(s) \mid \mathcal{F}_{t-1}] &= v(s-1)^\top \, \mathbb{E}[\tilde{\mathbf{W}}_2(t)^\top \tilde{\mathbf{W}}_2(t)] \, v(s-1) \\
&\leq \lambda_2(2) \, v(s-1)^\top v(s-1),    
\end{align*}
and by repeatedly conditioning, we have
\[
\mathbb{E}[v(s)^\top v(s)] \leq \lambda_2(2)^{t - s} \, \mathbb{E}[\|\mathbf{S}(s)\|_2^2].
\]
Since $\mathbb{E}[\|\mathbf{S}(s)\|_2^2] = \mathbb{E}\left[\sum_{k=1}^n [W_k(s) - W_k(s-1)]^2 X_k^2\right],$ we only need to bound \(\mathbb{E}[|W_k(s) - W_k(s-1)|^2]\). Recall that \(W_k(s) = {\mathbb{I}_{\{R_k(s) \in I_{n, \alpha}\}}}/{c_{n, \alpha}}\), so
\[
\mathbb{E}[|W_k(s) - W_k(s-1)|^2] = {\mathbb{E}[|W_k(s) - W_k(s-1)|]}/c_{n, \alpha},
\]
where we used the fact that indicator functions take values in \(\{0, 1\}\). Using the result from the previous lemma, we obtain
\[
\mathbb{E}[|W_k(s) - W_k(s-1)|] \leq  \frac{2C_k}{cs},
\]
where \(C_k = \frac{3\sigma_n^2(r_k)}{\gamma_k^2(1 - 2\alpha)}\) . Denoting $\tilde C_k =  \frac{\sqrt{6}\sigma_n(r_k)}{\sqrt{c}\gamma_k(1 - 2\alpha)}$,  we obtain,
\[
\mathbb{E}[\|\mathbf{S}(s)\|^2] \leq \frac{1}{cs}\sum_{k=1}^n \frac{2C_k}{c_{n, \alpha}} X_k^2 = \frac{1}{s}\cdot  \|\tilde {\mathbf C}\odot \mathbf X\|^2.
\]
Denoting $\tilde c = 1-\tilde \lambda$ with $\tilde \lambda = \sqrt{\lambda_2(2)}$, we obtain
{ 
\begin{align*}
   \sum_{s=1}^{t} \sqrt{\mathbb{E}\| \tilde{\mathbf{W}}_2(t{:}s) \, \mathbf{S}(s) \|^2_2} 
     &\leq \sum_{s=1}^{t}\frac{\sqrt{\lambda_2(2)}^{t - s}}{\sqrt{s}}\cdot  \|\tilde {\mathbf C}\odot \mathbf X\| \\
     &\leq \|\tilde {\mathbf C}\odot \mathbf X\| \cdot \sum_{s=1}^{t}\frac{e^{-\tilde c(t - s)}}{\sqrt{s}} .
\end{align*}
}
To bound the sum, similarly to the proof of Theorem 3 in \cite{van2025robust}, we decompose the sum using an intermediate time step $T = t - \log(\sqrt{t})/\tilde c$. Define \(T^* = \min \left\{ t > 1 \,\middle|\, \tilde c t > \log\left(t\right)/2 \right\}\). Then, for all \(t > T^*\), we have
{\setlength{\jot}{2pt}
\begin{align*}
    \sum_{s=1}^{t}\frac{e^{-\tilde c(t - s)}}{\sqrt{s}} &\leq e^{-\tilde c(t - T)}\sum_{s=1}^{T}{e^{-\tilde c(T - s)}} + \sum_{s=T+1}^{t}\frac{e^{-\tilde c(t - s)}}{\sqrt{T}} \\
     &\leq \frac{1}{\sqrt{\tilde c}}\left( \frac{1}{\sqrt{\tilde ct}} + \frac{1}{\sqrt{\tilde ct - \log(t)/2}}\right).
\end{align*}
}
All in all, using the expression of $\mathbf K$ concludes the proof.
\end{proof}

\subsection{Proof of \cref{thm:adaptive-gotrim}}

\begin{proof}
First, we use the relationship between the $\ell_1$ and $\ell_2$ norms: $\|\cdot\|_2 \leq\|\cdot\|_1.$  Then, we bound as follows: 
\begin{align*}
&\left\|\boldsymbol{Z}(t) - \bar{x}_{\alpha} \boldsymbol{1}_n \right\|_1 
=  \sum_{k=1}^n \left|\frac{N_k(t)}{D_k(t)} - \bar{x}_{\alpha}\right| \\
&= \sum_{k=1}^n \left| \frac{N_k(t) - \bar{x}_{\alpha}}{\max(1, M_k(t))} + \frac{(1 - M_k(t)) \, \mathbb{I}_{\{M_k(t) \geq 1\}} \, \bar{x}_{\alpha}}{\max(1, M_k(t))} \right| \\
&\leq \left\| \boldsymbol{N}(t) - \bar{x}_{\alpha} \boldsymbol{1}_n \right\|_1 + |\bar{x}_{\alpha}| \cdot \left\| \boldsymbol{M}(t) - \boldsymbol{1}_n \right\|_1.
\end{align*}
The convergence of $\mathbf  M(t)$ follows from similar arguments to the convergence of $\mathbf N(t)$, see proof of \cref{thm:original-gotrim}. Using that $\|\cdot\|_1 \leq \sqrt{n} \|\cdot\|_2$, we obtain the desired bound.
\end{proof}

\bibliography{main}





\end{document}